\documentclass{article}
\usepackage{ijcai25}

\usepackage{times}
\usepackage{soul}
\usepackage{url}
\usepackage[hidelinks]{hyperref}
\usepackage[utf8]{inputenc}
\usepackage[small]{caption}
\usepackage{graphicx}
\usepackage{amsmath}
\usepackage{amsthm}
\usepackage{booktabs}
\usepackage{algorithm}
\usepackage{algorithmic}
\usepackage[switch]{lineno}
\usepackage{subcaption}  
\usepackage{amssymb}
\usepackage{tabularx} 

\newtheorem{theorem}{Theorem}

\title{Approximated Behavioral Metric-based State Projection \\for Federated Reinforcement Learning}

\author{
Zengxia Guo$^{1,2}$
\and
Bohui An$^{1,2}$\and
Zhongqi Lu$^{1,2}$\thanks{Corresponding author.}\\
\affiliations
$^1$College of Artificial Intelligence, China University of Petroleum-Beijing, China\\
$^2$Hainan Institute of China University of Petroleum (Beijing),  Sanya, Hainan, China\\
\emails
2024211272@student.cup.edu.cn,
2024211271@student.cup.edu.cn,
zhongqi@cup.edu.cn
}

\begin{document}

\maketitle

\begin{abstract}
Federated reinforcement learning (FRL) methods usually share the encrypted local state or policy information and help each client to learn from others while preserving everyone's privacy.
In this work, we propose that sharing the approximated behavior metric-based state projection function is a promising way to enhance the performance of FRL and concurrently provides an effective protection of sensitive information.
We introduce FedRAG, a FRL framework to learn a computationally practical projection function of states for each client and aggregating the parameters of projection functions at a central server. 
The FedRAG approach shares no sensitive task-specific information, yet provides information gain for each client.
We conduct extensive experiments on the DeepMind Control Suite to demonstrate insightful results.
\end{abstract}

\section{Introduction}

In recent years, federated learning has emerged as a new approach to enable data owners to collaboratively train each one's improved local model with the help of the privacy preserved information from others~\cite{10.1145/3298981,federatedlearning,li2020federated2,wei2020federated,lyu2020threatsfederatedlearningsurvey}. Federated reinforcement learning (FRL) applies federated learning principles to reinforcement learning~\cite{zhuo2019federated}. In FRL, multiple clients, each with their own local environments, collaborate to learn a collective optimal policy~\cite{qi2021federated}.

Aggregating knowledge from clients in non-identical environments allows FRL to explore a huge state-action space, enhance sample efficiency and accelerate the learning process~\cite{wang2020optimizing}. However, FRL faces unique challenges primarily due to the different local environments and diverse data distributions among clients. In FRL, clients may experience very different states and rewards in their own environment, resulting in diverse data distribution. This diversity may lead to significant differences in the learning model, making it difficult for clients to converge to a robust common policy~\cite{zhao2018federated}.  Additionally, FRL must ensure that sensitive information remains protected from exposure to other clients or the central server~\cite{zhu2019deepleakagegradients,DBLP:journals/corr/abs-2103-06473}.

Previous researches found that learning representation based behavioral metric can significantly accelerate the reinforcement learning process and enhance the generality of policy~\cite{DBLP:journals/corr/abs-2006-10742,agarwal2021contrastive,kemertas2021towards}. This method involves learning a state projection function by evaluating the behavioral similarities between states, which are measured in terms of rewards and state transition probabilities. The state projection function is valuable to the learning process, yet it does not reveal any sensitive task-specific information. In the FRL settings, clients would not directly share the rewards and state information because of the privacy issues. Therefore, sharing the parameters of the state projection function could be a promising research direction for FRL.

\begin{figure}[t]
\centering
\includegraphics[width=0.5\textwidth]{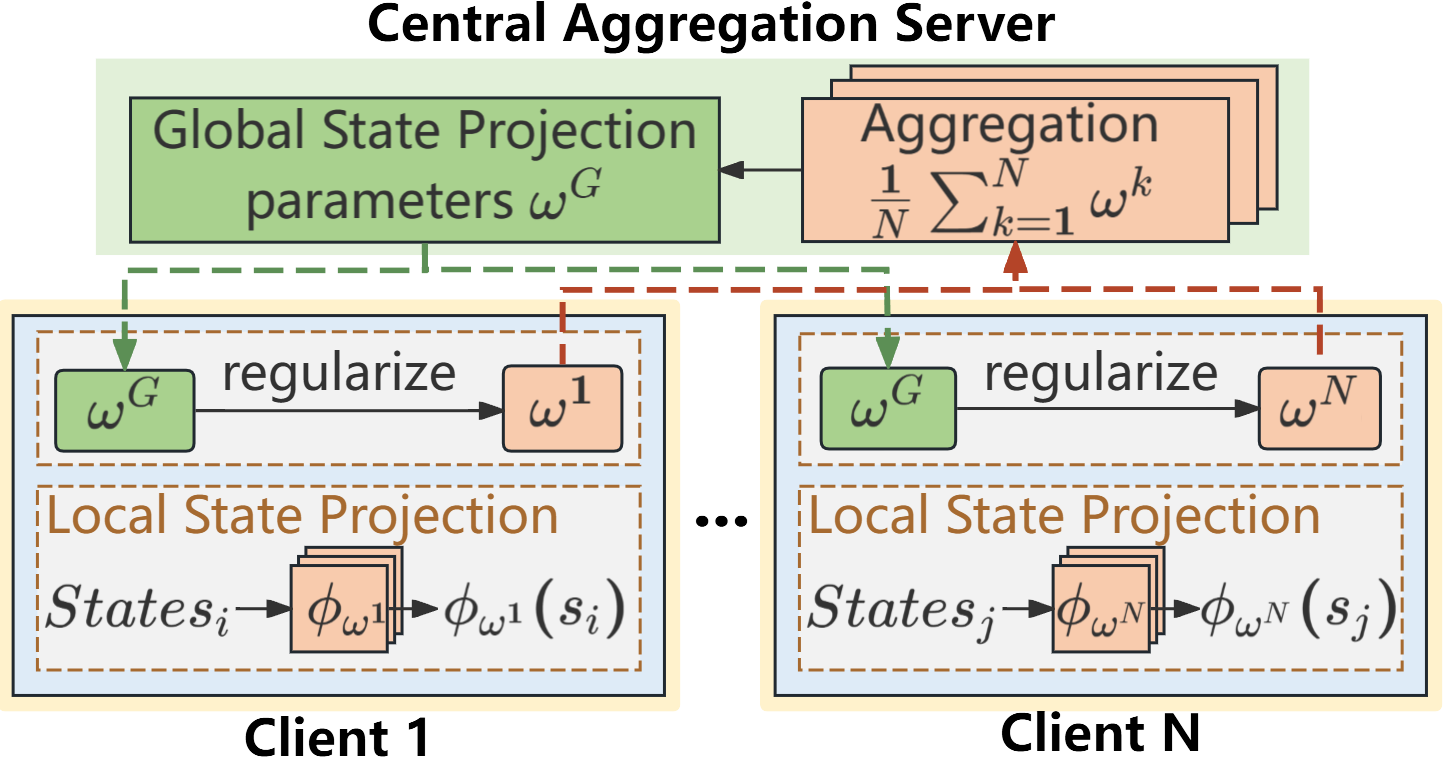} 
\caption{Framework of FedRAG. Periodically, the local state projection function parameters \(\omega^k\) are synchronized to a central server. Then the central server distributes the averaged parameters to the clients. For each client, a regularization term is incorporated to ensure that the client's local state projection parameters follow the global updates.}
\vspace{-1.5em}  
\label{fig:fig1}
\end{figure}

In this work, we propose the Federated Reinforcement Learning with Reducing Approximation Gap (FedRAG), a novel FRL framework to share parameters of state projection functions and to learn a local behavioral metric-based state projection function for each client.
We detail FedRAG's network architecture in Figure~\ref{fig:fig1}, emphasizing how client collaboration is achieved through shared state projection functions. The global state projection function is formed by aggregating local state projection functions, each trained with behavioral metrics to capture the unique transition dynamics and rewards of its respective environment. By integrating these locally learned features, the global state projection function reflects the diverse dynamics and rewards across different environments. Periodically, each client's local state projection function is replaced with the global state projection function, while the L2 regularization is continuously applied to maintain alignment throughout the learning process. Together, these mechanisms improve local state projection function and strategies that are robust and adaptable across varied environments.

The main contributions are as follows:
\begin{itemize}
\item We propose FedRAG, a novel federated reinforcement learning framework to share the projection function of states, instead of traditionally sharing the encrypted states information. Subsequent analysis show that our method is beneficial to privacy-preserving as a side-effect.
\item Under the FedRAG framework, we introduce a behavioral metric-based state projection function and develop its practical approximation algorithm in Federated Learning settings. Empirical results demonstrate our method is effective.
\end{itemize}

\section{Related Work}

\paragraph{Federated Learning.} Federated Learning (FL) was first introduced in FedAvg~\cite{mcmahan2017communication}, where training data remains distributed across mobile devices, and a shared model is learned by aggregating locally computed updates through iterative model averaging. Subsequently, FedProx~\cite{li2020federated} addresses system heterogeneity and statistical variability in federated networks. It incorporates a proximal term into local optimizations, allowing for variable computational efforts across devices, which helps stabilize diverse local updates.
To accommodate the inherent heterogeneity in FL, Per-FedAvg~\cite{fallah2020personalized} was developed as a personalized approach. This method adapts Model-Agnostic Meta-Learning (MAML) to provide a suitable initial model that quickly adapts to each user's local data after training. Another innovation, pFedMe~\cite{t2020personalized} tackles the statistical diversity among clients by utilizing Moreau envelopes as client-specific regularized loss functions, effectively decoupling personalized model optimization from global model learning.
\nocite{reddi2020adaptive}

\paragraph{Federated Representation Learning.} Recently, federated representation learning, which focuses on training models to extract effective feature representations directly from raw data, has become increasingly popular.
LG-FedAvg~\cite{liang2020think} optimizes for compact local representations on each device alongside a global model spanning all devices.
FedRep~\cite{collins2021exploiting} learns a shared data representation among clients while maintaining unique local heads to enhance each client's model quality.
Model Contrastive Learning (MOON)~\cite{li2021model} improves local update consistency by maximizing alignment between representations learned from local and global models.
Additionally, the Federated Prototype-wise Contrastive Learning (FedPCL) approach~\cite{tan2022federated} was introduced, leveraging pre-trained neural networks as backbones to enable knowledge sharing through class prototypes while constructing client-specific representations using prototype-wise contrastive learning.
FedCA~\cite{zhang2023federated} aggregates representations from each client, aligning them with a base model trained on public data to mitigate inconsistencies and misalignment in the representation space across clients.
TurboSVM-FL~\cite{wang2024turbosvm} accelerates convergence in federated classification tasks by employing support vector machines for selective aggregation and applying max-margin spread-out regularization on class embeddings.
Despite these advancements, research in federated representation learning specific to reinforcement learning remains limited.
\paragraph{Federated Reinforcement Learning.}
Federated Reinforcement Learning enables clients to collaboratively learn a unified policy while preserving privacy by avoiding the exchange of raw trajectories. FedPG-BR~\cite{fan2021fault} addresses convergence and fault tolerance against adversarial attacks or random failures in homogeneous environments using variance-reduced policy gradients. However, it does not tackle the challenges of heterogeneous environments, which is the focus of our work.
To address environmental heterogeneity, Jin \textit{et al.}~\shortcite{jin2022federated} introduced QAvg and PAvg algorithms, employing value function-based and policy gradient methods. They further proposed personalized policies that embed environment-specific state transitions into low-dimensional vectors, improving generalization and efficiency.
Similarly, Tang \textit{et al.}~\shortcite{tang2022fesac} developed FeSAC, based on the soft actor-critic framework, which isolates local policies from global integration and employs trend models to adapt to regional disparities.
Building on these advancements, our work focuses on learning a federated behavioral metric-based state projection function to effectively generalize across diverse environments. 
\nocite{fan2023fedhql}

\paragraph{Behavioral Metrics-based Representation Learning.}
Behavioral metric-based representation learning aims to create an embedding space that preserves behavioral similarities based on transitions and immediate rewards.
Bisimulation metrics~\cite{ferns2011bisimulation} measure state behavioral similarities in probabilistic transition systems for continuous state-space Markov Decision Processes (MDPs). On-policy bisimulation metrics~\cite{castro2020scalable} focus on behaviors specific to a given policy \(\pi\), incorporating a reward difference term and the Wasserstein distance between dynamics models.
To address the computational challenges associated with the Wasserstein distance, the MICo distance~\cite{castro2021mico} was developed to compare dynamics model distributions by measuring the distance between sampled subsequent states. The Conservative State-Action Discrepancy~\cite{article} separates the learning of the RL policy from the metric itself, focusing on the most divergent reward outcomes between states taking the same actions to define similarity in the embedding space.
Chen and Pan~\shortcite{chen2022learning} propose the Reducing Approximation Gap distance to recursively measure expected states over dynamics models, focusing on sampling from the policy \(\pi\) rather than the dynamics models. This approach reduces approximation errors and is particularly effective for representation learning. In our work, we apply approximation behavior metric-based representation learning to develop local state projection functions, capturing task-relevant behavioral similarities within each client's environment. Federated Learning then allows for sharing the parameters of these local projection function, enabling clients to benefit from generalized state representations across diverse environments.

\section{Preliminaries}
This section highlights the Federated Soft Actor-Critic (FeSAC) variant central to our research.
Soft Actor-Critic (SAC) is an off-policy actor-critic algorithm based on the maximum entropy reinforcement learning framework~\cite{haarnoja2018soft}. It aims to maximize cumulative future rewards and entropy to enhance robustness and exploration while preventing convergence to suboptimal policies.
FeSAC extends SAC to a federated setting, enabling collaborative training among clients operating in diverse environments while ensuring data privacy.
The global environment \( E = \{ E^1, E^2, \ldots, E^N \} \) is composed of \( N \) distinct local environments, and each client \( k \) operates within its own unique local environment \( E^k \). The transition probabilities differ across local environments, i.e., 
\(P(s^i_{t+1} | s^i_{t}, a) \neq P(s^j_{t+1} | s^j_{t}, a)\), \(i \neq j\).

As our study focuses on applying approximated behavioral metric-based representation learning to FRL, we introduce the state projection function when discussing FeSAC.
In the scope of representation learning for deep RL, a state projection function \(\phi_{\omega^k}\) maps a high-dimensional state to low-dimensional vector, from which the policy \( \pi_{\psi^k}(a|\phi_{\omega^k}(s))\) is learned.
We configure all critic networks, target critic networks, and action networks to take the state representation \(  \phi_ {\omega^k} (s) \) as input instead of the raw state $s$.

Unlike traditional FRL, the objective of FeSAC is to derive a set of maximum entropy policies that are specifically optimized for their respective local environments. The target policy \(\widetilde{\pi}^k\) for client \(k\) in its local environment \(E^k\) is 
\vspace{-0.45em}
\begin{align}
\widetilde{\pi}^k = & \arg\max_{\pi^k} \sum_{t=0}^{T} \mathbb{E}_{(s_{t}^k, a_{t}^k) \sim \tau_{\pi^k}} \Big[\, \gamma^t r(s^k_t, a^k_t) \nonumber \\
& \quad + \alpha^k \mathcal{H}\big(\pi^k(\cdot \mid \phi_{\omega^k}(s_{t}^k))\big) \Big],
\end{align}
where \(s_t^k\) and  \(a_t^k\) represent the state and action made by client \(k\) in its local environment \(E^k\) at time \(t\); \(\tau_{\pi^k}\) refers to the trajectory generated by the policy \(\pi^k\) of client \(k\), which encompasses the sequence of states and actions over time; \(\gamma^k\) is the discount rate; \(\alpha^k\) is the entropy regularization coefficient used to control the importance of entropy; \(\mathcal{H}(\pi^k(\cdot | \phi_{\omega^k}{(s_{t}^k)}))=\mathbb{E}[-log\pi^k(\cdot | \phi_{\omega^k}{(s_{t}^k}))]\) represents the entropy of the policy.

To evaluate the impact of the policy on local environments, the soft state value is defined as:
\begin{align}
V(s_t^k) = & \underset{ a_t^k \sim \pi_{\psi^k}}{\mathbb{E}} \left[ Q_{\theta^k}(\phi_{\omega^k}(s^k_t), a^k_t) \right. \nonumber \\
& \quad - \left. \alpha^k \log \pi_{\psi^k}(a^k_t|\phi_{\omega^k}(s^k_t)) \right],
\end{align}
where \(Q_{\theta^k}\) denote the local critic Q network for client \(k\).

Each client adjusts its local Q-network to approximate the global Q-network, thus leveraging global knowledge while retaining its own characteristics:
\vspace{-1em}
\begin{align}
L_{Q}(\theta^{k}) = & \mathbb{E}_{(s_t^k, a_t^k, r_t^k, s_{t+1}^k) \sim \mathcal{D}^k} \Bigg[ Q_{\theta^k}\big(\phi_{\omega^k}(s_t^k), a_t^k\big) \notag \\
& \quad - \left( r_t^k + \gamma V_{\bar{\theta}}(s_{t+1}^k) \right) \Bigg]^2,
\label{eq:critic_loss_fesac}
\end{align}
where \(V_{\bar{\theta}}\) denotes use the target critic Q networks to calculate the soft state value.

In FeSAC, the target critic Q network refers to the global critic Q network, which is broadcasted by the server to all clients. The global critic Q network \(Q_{\bar{\theta}}\) is formed by aggregating the local critic Q networks of each client through soft updates, considering the reward differences of state-action pairs in each client's environment to obtain a value estimation in a global context:
\begin{equation}
Q_{\bar{\theta}} \leftarrow \epsilon Q_{\theta^k} + (1 - \epsilon) Q_{\bar{\theta}}, \quad k \in \{1, 2, \ldots, N\},
\end{equation}
where \(\epsilon\) is the aggregation factor. 

The updated local Q-network then guides the update of the local policy, which keeps the local variability as well as learning the implicit trend of the global environment:
\begin{align}
\label{eq:actor_loss_fesac}
& L_{\pi}(\psi^k) = \underset{s_t^k \sim \mathcal{D}^k}{\mathbb{E}} \bigg [ \underset{a_t^k \sim \pi_{\psi^k}(\cdot|\phi_{\omega^k}(s_t^k))}{\mathbb{E}} \left[ \right. \nonumber \\ &  \left.
 \alpha^k \log \pi_{\psi^k}(a_t^k \mid \phi_{\omega^k}(s_t^k))
 - Q_{\theta^k}(\phi_{\omega^k}(s_t^k), a_t^k) \right] \bigg ].
\end{align}

The temperature parameter \(\alpha^k\) is adapted to balance exploration and exploitation by controlling the relative importance of the entropy term in the policy's objective. The update objective for \(\alpha^k\) in client \(k\) is as follows~\cite{haarnoja2018soft2}:
\begin{align}
\label{eq:alpha_loss}
L_{\alpha}(\alpha^k) = & \underset{s^k_t\sim \mathcal{D}^k}{\mathbb{E}}\left[\mathbb{E}_{a^k_t \sim \pi_{\psi^k}(\cdot|\phi_{\omega^k}(s_t^k))} \right. \nonumber \\ & \left. [\alpha^k \log \pi_{\psi^k}(a^k_t | \phi_{\omega^k}(s_t^k)) - \alpha^k \bar{\mathcal{H}} ]\right],
\end{align}
where \(\bar{\mathcal{H}}\) is a target entropy level to tune the degree of exploration and \(\bar{\mathcal{H}}=-|\mathcal{A}|\).

\section{Methodology}
In this section, we present the problem formulation for federated reinforcement learning with heterogeneous environments, introduce the approximated behavioral metric-based state projection function, propose the FedRAG framework and provide a theoretical analysis of its privacy preserving.
\subsection{Problem Formulation}
In federated reinforcement learning with heterogeneous environments, \(N\) clients each interact with their own unique local environment \(E^k\), each modeled as a unique Markov Decision Process (MDP): \(\{ S^k,A,R^k,P^k,\gamma\}\). Each client has a unique state space \(S^k\), reward function \(R^k(s,a)\), and state transition dynamics $P^k(s’|s,a)$, reflecting the diversity of their environments, while sharing a common action space \(A\) and discount factor \(\gamma\).
A central server facilitates collaboration by periodically aggregating and distributing shared model parameters, specifically the state projection function 
\(\phi_{\omega}\) in FedRAG. This function maps local states to a shared embedding space, enabling clients to benefit from collective learning while preserving privacy. FedRAG optimizes local policies \(\pi^k \left (a|\phi_{\omega^k}\left (s\right )\right )\) by sharing the parameters of \(\phi_\omega\), aiming to maximize cumulative reward and entropy:

\vspace{-0.5em}
\begin{align}
\widetilde{\pi}^k = & \arg\max_{\pi^k} \frac{1}{N} \sum_{k=1}^{N} \Bigg\{  
    \sum_{t=0}^{\infty} \mathbb{E}_{(s_{t}^k, a_{t}^k) \sim \tau_{\pi^k}} \Bigg[ \gamma^t R^k\left (s^k_t, a^k_t\right ) \nonumber\\
    &\quad + \alpha^k \mathcal{H}\left (\pi_{\psi^k}\left (\cdot \mid \phi_{\omega^k}\left (s_{t}^k\right)\right )\right) \Bigg] 
\Bigg\},
\end{align}
where \(a^k_t\sim \pi^k(\cdotp |\phi_{\omega^k}(s^k_t))\), \(s^k_{t+1}\sim P^k(\cdotp |s^k_t,a^k_t)\). To preserve data privacy, only the parameters of the state projection function \(\omega\) are shared between clients and the server, while raw states, rewards, and transition dynamics remain local to each client, ensuring sensitive information is not exchanged while enabling effective federated learning.

\subsection{Client RAG Distance}
In FeSAC, clients across different environments share knowledge by aligning their local Q networks with the global Q network, enabling optimal local policies while adapting to changes. However, in complex environments, clients may struggle to capture task-relevant information (see Section~\ref{comparison}), leading to unclear global perceptions and hindering adaptation to environmental changes.

To enhance generalization in complex environments, we introduce behavior metric-based representation learning. This approach learns robust state representations that filter out task-irrelevant background information, speeding up the learning process and improving policy generalization across diverse environments.

For each client \(k\), behavioral metric-based representation learning is to learn a local state encoding network ${\phi}_{\omega^k}:S^k\to\mathbb{R}^n$ with parameters $\omega^k$, which can be cast as a minimization problem of the loss between the distance on the embedding space, $\hat{d}({\phi }_{\omega^k }\left ( {s}_{i}^k\right ),{\phi }_{\omega^k }\left ( {s}_{j}^k\right ))$, and the corresponding behavior metric, ${d}^\pi({s}^k_{i},{s}^k_{j})$ , between any pair of states $s_i^k$ and $s_j^k$:
\vspace{-0.2em}
\begin{equation}
L_{\phi}({\omega^k}) = \mathbb{E}\left[ \left( \hat{d}(\phi_{\omega^k}(s_{i}^k), \phi_{\omega^k}(s_{j}^k)) - d^\pi(s_{i}^k, s_{j}^k) \right)^2 \right].
\label{eq:BehaviorMetric}
\end{equation}

The Reducing Approximation Gap (RAG) distance is a behavioral metric that measures the absolute difference between the reward expectations of two states and the distance between the next state expectations of dynamics models. And it is defined as follows:
\begin{align}
d^\pi(s_i^k, s_j^k) = & \left| \mathbb{E}_{a_i^k \sim \pi^k} r^{a_i^k}_{s_i^k} - \mathbb{E}_{a_j^k \sim \pi^k} r^{a_j^k}_{s_j^k} \right| \nonumber \\ & + \gamma \mathbb{E}_{a_i^k \sim \pi^k, a_j^k \sim \pi^k} d^\pi(\mathbb{E}[s^k_{i+1}], \mathbb{E}[s^k_{j+1}]),
\end{align}
where \( \mathbb{E}_{a_i^k \sim \pi^k} r^{a_i^k}_{s_i^k}\) represents the expected reward obtained by taking action \(a^k_i\) in state \(s_i^k\) under the policy \(\pi^k\) of client \(k\), $\mathbb{E}[s^k_{i+1}] = \mathbb{E}_{s^k_{i+1} \sim P^{a^k_i}_{s^k_i}}[s^k_{i+1}]$ is the expectation value of next state over the dynamics model $P({s^k_i}, {a^k_i})$.

Then the approximation of RAG relax the computationally intractable reward difference term without introducing any approximate gap, as shown below:
\begin{align}
& d^\pi (s^k_i, s^k_j)\nonumber \\= &\ \sqrt{ \mathbb{E}_{a^k_i \sim \pi^k, a^k_j \sim \pi^k} \left[ \left( r^{a_i^k}_{s_i^k} - r^{a_j^k}_{s_j^k} \right)^2 \right]  - \text{Var}[r_{s^k_i}] - \text{Var}[r_{s^k_j}] } \nonumber\\
+ &\ \gamma \mathbb{E}_{a^k_i \sim \pi^k, a_j^k \sim \pi^k} d^\pi \left( \mathbb{E}[s^k_{i+1}], \mathbb{E}[s^k_{j+1}] \right).
\label{eq:ApproximateRAP}
\end{align}

Since the reward variance $\text{Var}[r_{s^k_i}]$ is computationally intractable, we can learn a neural network approximator to estimate it by assuming that the reward $r_{s^k}$ on state $s^k$ is Gaussian distributed. Let $\hat{R}_{\xi^k}(s^k) = \{\hat{\mu}(r_{s^k}), \hat{\sigma}(r_{s^k})\}$ be the learned reward function approximation parameterized by $\xi^k$, which outputs a Gaussian distribution. The loss function is:
\begin{equation}
\label{eq:reward_loss}
L_{\hat{R}}(\xi^k) = \mathbb{E}_{(s^k, r^k) \sim \mathcal{D}^k} \left[ \left (\frac{r^k - \hat{\mu}(r_{s^k})}{2\hat{\sigma}(r_{s^k})} \right )^2\right],
\end{equation}
where $\hat{\mu}$ and $\hat{\sigma}$ are the mean and the standard deviation, respectively.

Similarly, in order to estimate the expected next states  $\mathbb{E}[s^k_{i+1}]$, we learn a dynamics model $\hat{P}_{\eta^k}(\phi_{\omega^k}(s), a) = \{\hat{\mu}(\hat{P}_{\phi_{\omega^k}(s)}^a), \hat{\sigma}(\hat{P}_{\phi_{\omega^k}(s)}^a)\}$ for each client, which outputs a Gaussian distribution over the next state embedding: 
\begin{equation}
\label{eq:dynamics_loss}
L_{\hat{P}}(\eta^k) = \underset{(s^k_i, a^k_i, s^k_{i+1}) \sim \mathcal{D}^k}{\mathbb{E}} \left[ \left(\frac{ \phi_{\omega^k}(s^k_{i+1}) \!-\! \hat{\mu}\left (\hat{P}^{a^k_i}_{\phi_{\omega^k}({s^k_i})} \right )}{2 \hat{\sigma}\left (\hat{P}^{a^k_i}_{ \phi_{\omega^k}(s^k_i)}\right)}  \right)^2\right].
\end{equation}

Based on the above approximation, the RAG loss for each client can be defined as: 
\begin{equation}
\begin{split}
& L_{\text{RAG}} (\phi_{\omega^k}) =  \mathbb{E}_{\mathcal{D}^k} \\ & \bigg[ \left( \hat{d}\left(\phi_{\omega^k}(s^k_i), \phi_{\omega^k}(s^k_j)\right) \!-\! \gamma \hat{d}(\hat{\mu}(\hat{P}^{a^k_i}_{ \phi_{\omega^k}(s^k_i)}), \hat{\mu}(\hat{P}^{a^k_j}_{ \phi_{\omega^k}(s^k_j)})) \right)^2 \\
 & -  \left(\Big| r^{a^k_i}_{s^k_i} - r^{a^k_j}_{s^k_j} \Big|^2 - (\hat{\sigma}(r_{s^k_i}))^2 -(\hat{\sigma}(r_{s^k_j}))^2 \right)\bigg]^2,
\end{split}
\label{eq:RAG_Loss}
\end{equation}
where \(\mathcal{D}^k\) represents the replay buffer or the set of data collected from environment \(E^k\) by the RL algorithm, e.g. SAC. 

Considering that the behavior metric has non-zero self-distance, the distance on the embedding space adopts the approximate form proposed in MICo~\cite{castro2021mico}, which produces a non-zero self-distance and helps in maintaining proximity between similar states rather than pushing them apart:
\begin{equation}
\hat{d}(\phi(s^k_i), \phi(s^k_j)) = \|\phi(s^k_i)\|^2 + \|\phi(s^k_j)\|^2 + K \varphi(\phi(s^k_i), \phi(s^k_j)),
\end{equation}
while $\varphi$ is absolute angle distance and $K$ is a hyper-parameter. 

\subsection{FedRAG Framework}
\begin{algorithm}[t]
\caption{FedRAG algorithm}
\label{alg:algorithm}
\begin{algorithmic}[1]
\STATE Initialize local networks $\phi_{\omega^k}$, $\phi_{\bar{\omega}^k}$, $Q_{\theta^k}$, $Q_{\bar{\theta}^k}$, $\pi_{\psi^k}$, $\hat{R}_{\xi^k}$, $\hat{P}_{\eta^k}$ for each client $k \in \{1, 2, \dots, N\}$, and global network $\phi_{\omega^G}$ at the server.
\STATE Synchronize local and global parameters: $\omega^k, \bar{\omega}^k \leftarrow \omega^G$ for each client $k$.
\STATE Initialize empty replay memory $\mathcal{D}^k$ for each client $k$.
\WHILE{running}
    \FOR{each client $k$}
        \STATE Observe state $s_t$ from local environment $E^k$, sample action $a_t \sim \pi(\cdotp|\phi_{\omega^k}(s_t))$ and execute.
        \STATE Receive reward $r_t\leftarrow R(s_t,a_t)$ and transition to next state $s_{t+1} \sim P(\cdotp | s_t, a_t)$.
        \STATE Store transition $(s_t, a_t, r_t, s_{t+1})$ in $\mathcal{D}^k$.
        \STATE Update local networks $Q_{\theta^k}$, $\pi_{\psi^k}$, $\alpha^k$, $\hat{P}_{\eta^k}$, $\hat{R}_{\xi^k}$, $\phi_{\omega^k}$ via gradient descent using Eq.~\ref{eq:critic_loss_fesac},\ref{eq:actor_loss_fesac},\ref{eq:alpha_loss},\ref{eq:dynamics_loss},\ref{eq:reward_loss},\ref{eq:FedRAG_Loss}
        \STATE Softly update target networks:
        $\bar{\theta}^k \leftarrow \tau_Q\theta^k + (1-\tau_Q)\bar{\theta}^k$, 
        $\bar{\omega}^k \leftarrow \tau_{\phi}\omega^k + (1-\tau_{\phi})\bar{\omega}^k$.
        \IF{running $n$ iterations}
            \STATE Upload $\omega^k$ to federated center node.
        \ENDIF
    \ENDFOR
    \IF{in federated center node}
        \STATE Aggregate global parameters: $\omega^G \leftarrow \frac{1}{N}\sum_{k=1}^N \omega^k$.
        \STATE Broadcast updated global parameters: $\omega^k, \bar{\omega}^k \leftarrow \omega^G$ for all clients.
    \ENDIF
\ENDWHILE
\end{algorithmic}
\end{algorithm}
\noindent Under the federated learning framework, we share the parameter \(\omega\) of the state projection function \(\phi_\omega\). The FedRAG framework operates with multiple clients and a federated central node. Each client \(k\) generates local parameters \(\omega^k\) for the state projection function and updates policy networks based on their local environment. The federated central node collects these local parameters \(\omega^k\) from all clients, aggregates them into a global distribution, and then distributes the updated global parameters back to the clients. Each client uses the state projection $ \phi_{\omega^k}(s)$ as input for both the actor and critic networks. We assume that global $\omega$ follows a Gaussian distribution,with each client learning only a portion of the overall distribution. Therefore, we add a Gaussian regularization term after the RAG regression function Eq.~\ref{eq:RAG_Loss}, leading to the new loss formulation:
\begin{align}
&L_{\text{FedRAG}} (\phi_{\omega^k}) = \mathbb{E}_{\mathcal{D}^k}\nonumber\\& \bigg[ \left( \hat{d}\left (\phi_{\omega^k}(s^k_i), \phi_{\omega^k}(s^k_j)\right )\!-\!\gamma \hat{d} (\hat{\mu} (\hat{P}^{a^k_i}_{ \phi_{\omega^k}(s^k_i)}), \hat{\mu}(\hat{P}^{a^k_j}_{ \phi_{\omega^k}(s^k_j)}) ) \right)^2 \nonumber\\
 \!-\! &\ \left(\Big| r^{a^k_i}_{s^k_i}\!-\!r^{a^k_j}_{s^k_j} \Big|^2\!-\!(\hat{\sigma}(r_{s^k_i}))^2\!-\!(\hat{\sigma}(r_{s^k_j}))^2 \right)\bigg]^2\!+\!\frac{\lambda}{2}\| \omega^k\!-\!\omega^G \|_2^2,
\label{eq:FedRAG_Loss}
\end{align}
where $\omega^G$ represents the expectation of the global Gaussian distribution.
The regularization term helps reduce environmental heterogeneity, thereby enhancing collaborative learning, as demonstrated in Section~\ref{heter}.

The proposed FedRAG is detailed in Algorithm~\ref{alg:algorithm}. During the FL process, we upload $\omega^k$ to the server periodically. According to the central limit theorem, we approximate the global Gaussian distribution by aggregating the mean of all local $\omega^k$ at the server. 
The server then distributes the results to each client, aligning local learning with the global distribution.
By averaging local state projection function parameters, FedRAG integrates the specialized features learned in each client’s environment. Each client can maintain its own local training advantages while incorporating the global nature, and perform better when dealing with data outside of its own. 

\subsection{Effectiveness of Anti-attack}
Note that the data we aim to protect is not directly uploaded, potentially reducing the need for additional privacy techniques such as differential privacy or homomorphic encryption. Below, we analyze the privacy-preserving properties of our approach.
\nocite{truex2019hybrid}
\nocite{mothukuri2021survey}
\nocite{li2021survey}
\nocite{abadi2016deep}
\nocite{geyer2017differentially}
\nocite{mohassel2018aby3}
\nocite{zhang2020batchcrypt}
\nocite{fang2021privacy}
\nocite{park2022privacy}

One of the major issues in federated learning is preserving privacy. In our analysis, we consider the existence of semi-honest adversaries. The adversaries may launch privacy attacks to snoop on the training data of other participants by analyzing periodic updates (e.g., gradients) of the joint model during training~\cite{zhu2019deepleakagegradients}. Such kind of attacks is referred to as Bayesian inference attack~\cite{zhang2022no}.

A Bayesian inference attack is an optimization process  that aims to infer the private variable $D_{k}$  to best fit client $k$ protected exposed information ${W}_{k}^{S}$ as 
\begin{equation}
\begin{split}
d^{*} &= \arg \max_{d}{\log({f}_{D_{k}|{W}_{k}^{S}}(d|w))} \\
&= \arg \max_{d}{\log(\frac{{f}_{{W}_{k}^{S}|D_{k}}(w|d){f}_{D_{k}}(d)}{{f}_{{W}_{k}^{S}}(w)})} \\
&= \arg \max_{d}[\log{f}_{{W}_{k}^{S}|D_{k}}(w|d)+\log{f}_{D_{k}(d)}]
\end{split}
\label{eq:Bayesian_inference_attack}
\end{equation}
where ${f}_{D_{k}|{W}_{k}^{S}}(d|w)$ is the posterior of $D_{k}$ given the protected variable ${W}_{k}^{S}$. According to Bayes's theorem, maximizing the log-posterior ${f}_{D_{k}|{W}_{k}^{S}}(d|w)$ on $D_{k}$  involves maximizing summation of $\log{({f}_{{W}_{k}^{S}|D_{k}}(d|w))}$ and $\log{({f}_{D_{k}}({d}))}$. The former one aims to find $D_{k}$ to best match ${W}_{k}^{S}$, and the latter one aims to make the prior of $D_{k}$ more significant. The learned conditional distribution ${f}_{D_{k}|{W}_{k}^{S}}$ from the Bayesian inference attack reflects the dependency between ${W}_{k}^{S}$ and $D_{k}$, which determines the amount of information that adversaries may infer about $D_{k}$ after observing ${W}_{k}^{S}$. However, in our approach, the parameter $\omega$ that we participate in federated learning is related to the representation function $\phi$ of the state. From the loss $L_{FedRAG} (\phi_{\omega})$ in Equation~\ref{eq:FedRAG_Loss}, we can also see that $\omega$ is only related to the mapped state and reward, and has nothing to do with our private data state. Therefore, our proposed FedRAG protects the privacy of local state information to a certain extent.

\section{Experiment}
\subsection{Experimental Settings}
In this section, we evaluate the effectiveness and generalization of FedRAG using DeepMind Control Suite (DMC). The DMC is a benchmark for control tasks in continuous action spaces with visual input~\cite{tassa2018deepmind}. 
We simulated different environments by modifying key physical parameters for several tasks: pole length (cartpole-swing), torso length (cheetah-run), finger distal length (finger-spin), and torso length (walker-walk). 

As described in the previous section, each client projects state observation to the embedding space by using the approximated behavioral metric-based local state projection network, and updates local SAC network for policy evaluation and improvement. 
We perform experiments on 2 settings: 1) \textbf{Local}: clients can only interact and update local network in their own different environments without information sharing; 2) \textbf{Federated}: clients interact with their respective environments, update local network with information sharing according to federated methods.

In our study, we render 84×84 pixels and stack 3 frames as observation at each time step. We set an episode to consist of 125 environment steps, training over a total of 4000 episodes, which equates to 500,000 steps. For each setting, we evaluate the performance of each clients in both the same and other environments every 16 local update episodes. In the federated learning scenario, every 4 episodes, clients upload their local parameters, which the server then aggregates and redistributes as global parameters.

\subsection{FedRAG vs. Baseline Performance Comparison}
\label{comparison}  
\begin{figure}[h]
    \vspace{-1em}
    \centering
    \begin{subfigure}{0.48\textwidth}
        \centering
        \includegraphics[width=\linewidth]{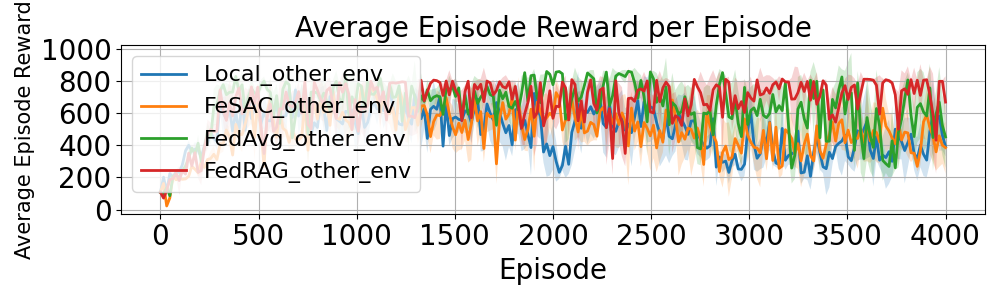}  
    \end{subfigure}
    \vspace{-0.8em}
    \begin{subfigure}{0.48\textwidth}
        \centering
        \includegraphics[width=\linewidth]{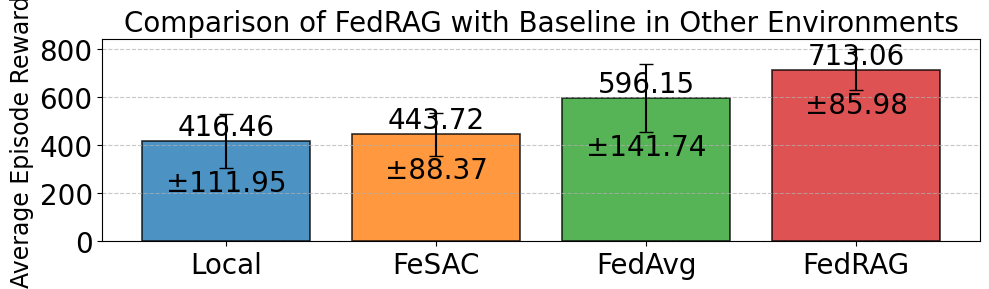}  
    \end{subfigure}
    \caption{Comparison of FedRAG with Baseline in other environments.}
    \vspace{-0.5em}  
    \label{fig:fedrag_performance3}
\end{figure}
\noindent As illustrated in Figure~\ref{fig:fedrag_performance3}, we compared our proposed FedRAG method (\(\lambda=0.001\)) with FedAvg (equivalent to FedRAG with \(\lambda=0\)), FeSAC, and Local methods in the CartPole task with varying pole lengths. We assessed the average episode reward and standard deviation achieved by the clients in other environments. The results show that clients in the Local group, trained exclusively in their own environment without federated learning, struggled to adapt to other environments, resulting in the lowest performance. FeSAC had limited effectiveness in capturing task-relevant information in complex states, leading to only modest performance improvements. In contrast, FedRAG outperformed FedAvg by effectively integrating the global state projection function during local updates, resulting in significant performance gains in other environments. 

Unlike traditional FRL methods, FedRAG transmits only lightweight state projection parameters, significantly reducing communication overhead. 
While maintaining the state projection function incurs minor computational cost, it accelerates the RL process and enhances model generalization, yielding an overall benefit.

\subsection{the Hyper-parameter of FedRAG}
\begin{figure}[h]
    \vspace{-1em}
    \centering
    \begin{subfigure}{0.48\textwidth}
        \centering
        \includegraphics[width=\linewidth]{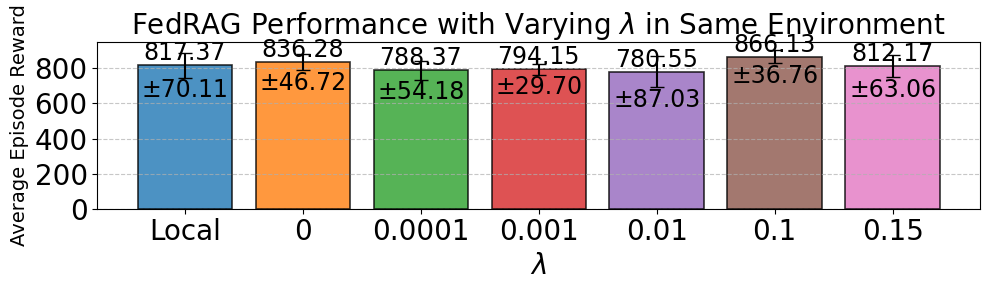}  

    \end{subfigure}
    \vspace{-1em}
    \begin{subfigure}{0.48\textwidth}
        \centering
        \includegraphics[width=\linewidth]{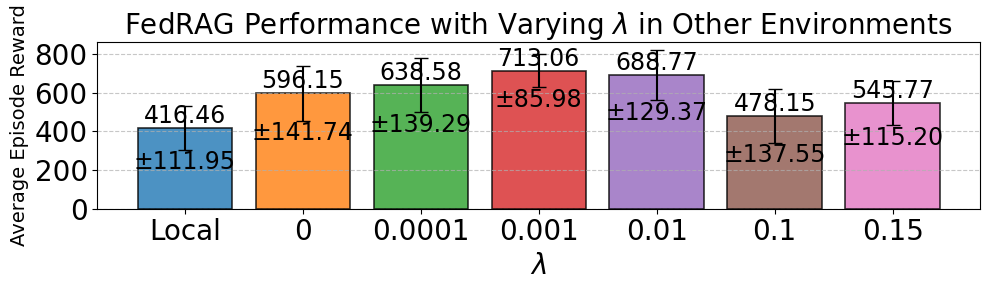}  
    \end{subfigure}
    \caption{The results of varying lambda. In the above experiment, the training data and testing data are from environments with same setting, while in the below experiment, they are come from environments with different settings.
    }
    \vspace{-0.5em}  
    \label{fig:fedrag_performance}
\end{figure}

In the local update process of FedRAG, the regularization term in Equation~\ref{eq:FedRAG_Loss} aligns the local state projection function with the global one. As shown in Figure~\ref{fig:fedrag_performance}, we evaluated the performance of the Local group and multiple FedRAG groups with varying $\lambda$ values in both same and other environments.

In other environments, increasing $\lambda$ improves the alignment between the local and global state projection functions, enhancing parameter sharing and boosting performance across other environments. However, a large $\lambda$ may keep local updates too close to their initial global state, restricting parameter updates and slowing convergence. The optimal performance was achieved at $\lambda = 0.001$. 

In the same environment, performance remained stable with minor fluctuations, highlighting the robustness of our approach. 
In FRL with heterogeneous environments, each local environment has unique state transition dynamics. 
The aggregated global state projection function integrates information from diverse environments, which can introduce conflicting or irrelevant gradients that do not benefit local performance.
While increasing $\lambda$ enhances generalization across diverse environments by leveraging shared global knowledge, it can limit local networks' ability to optimize for their specific environment. 
This trade-off reflects the balance in FL between generalization and specialization.

Overall, while performance remained stable in the same environment across all $\lambda$ values, notable improvements were observed in other environments, confirming the effectiveness of our federated approach. 

\subsection{Performance Improvement for Federated Learning}
\begin{figure}[h]
     \vspace{-0.5em}
    \centering
    \begin{subfigure}{0.48\textwidth}
    \vspace{-0.3em}
        \centering
        \includegraphics[width=\linewidth]{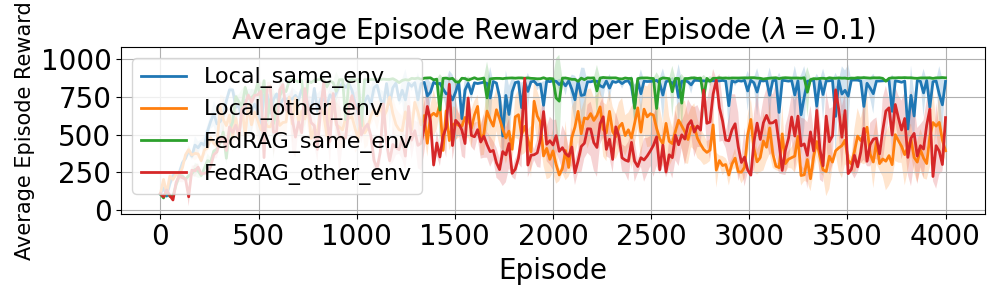}  
    \end{subfigure}
    \begin{subfigure}{0.48\textwidth}
        \centering
        \includegraphics[width=\linewidth]{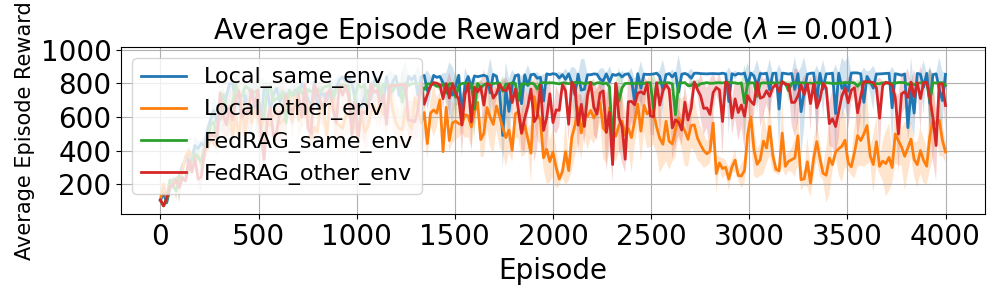}  
    \end{subfigure}
    \vspace{-2em}
    \caption{Comparison of Local and FedRAG with $\lambda=0.1/0.001$ in same or other environments.}
    \vspace{-0.5em}
    \label{fig:fedrag_performance2}
\end{figure}

In Figure~\ref{fig:fedrag_performance2}, we compare the performance of the FedRAG method ($\lambda=0.1/0.001$) with the Local approach by evaluating average episode rewards in both the same and other environments. The Local approach limits clients to their own environments, resulting in local optimal policies that poorly generalize. In contrast, FedRAG aggregates local state projection functions on a central server to create a global state projection function. By sharing this global function during local updates, clients benefit from cross-environment knowledge sharing while maintaining data privacy. With $\lambda=0.1$, FedRAG enhances local performance by leveraging shared knowledge to overcome local optima, while also improving performance in other environments. At $\lambda=0.001$, FedRAG achieves the best results in other environments with minimal loss in the same environment, demonstrating strong generalization and robustness across diverse settings.

\subsection{FedRAG Performance on Various DeepMind Control Tasks}
\begin{figure}[h]
    \centering
    \vspace{-0.8em}
    \begin{subfigure}{0.24\textwidth}
        \centering
        \includegraphics[width=\linewidth]{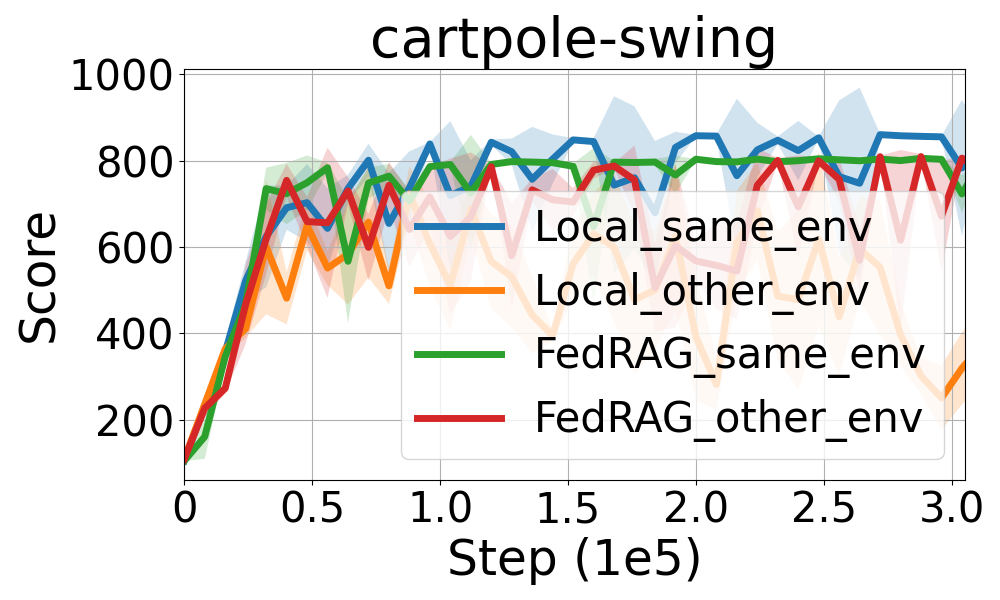}  
    \end{subfigure}
    \hspace{-0.5em}
    \begin{subfigure}{0.24\textwidth}
        \centering
        \includegraphics[width=\linewidth]{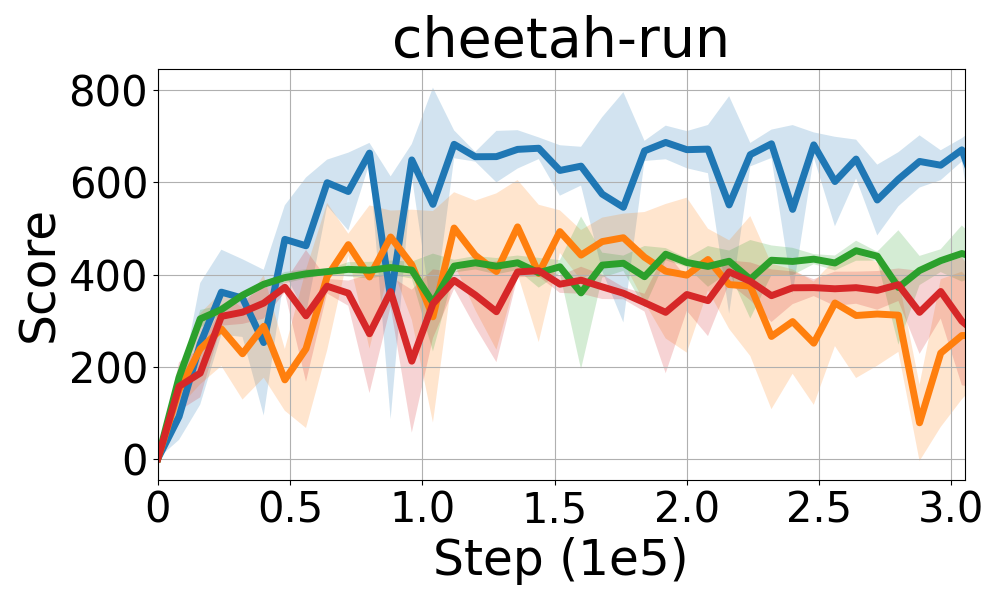}  
    \end{subfigure}
    \vspace{-0.5em}
    \begin{subfigure}{0.24\textwidth}
        \centering
        \includegraphics[width=\linewidth]{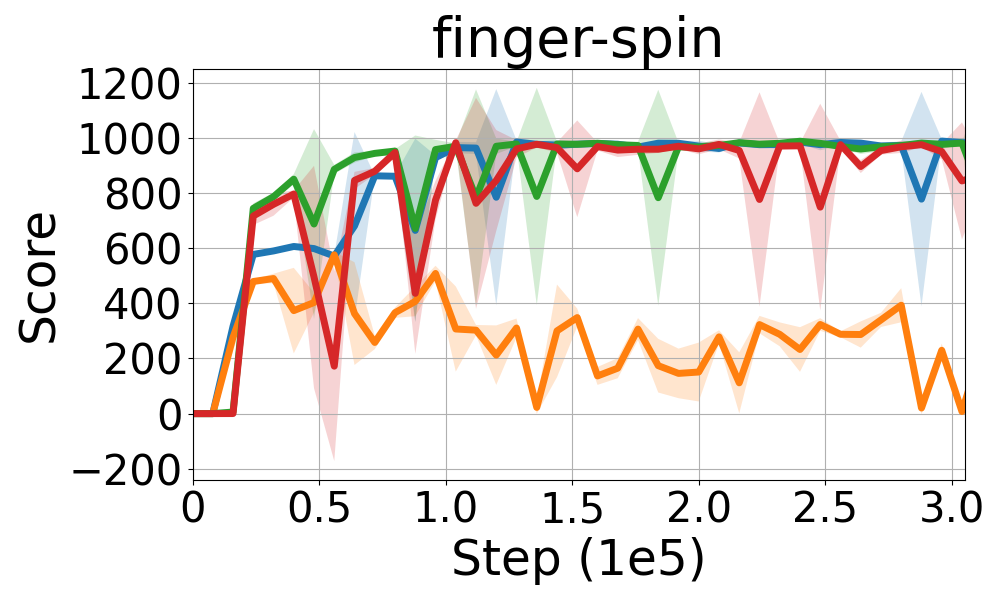}  
    \end{subfigure}
    \hspace{-0.5em}
    \begin{subfigure}{0.24\textwidth}
        \centering
        \includegraphics[width=\linewidth]{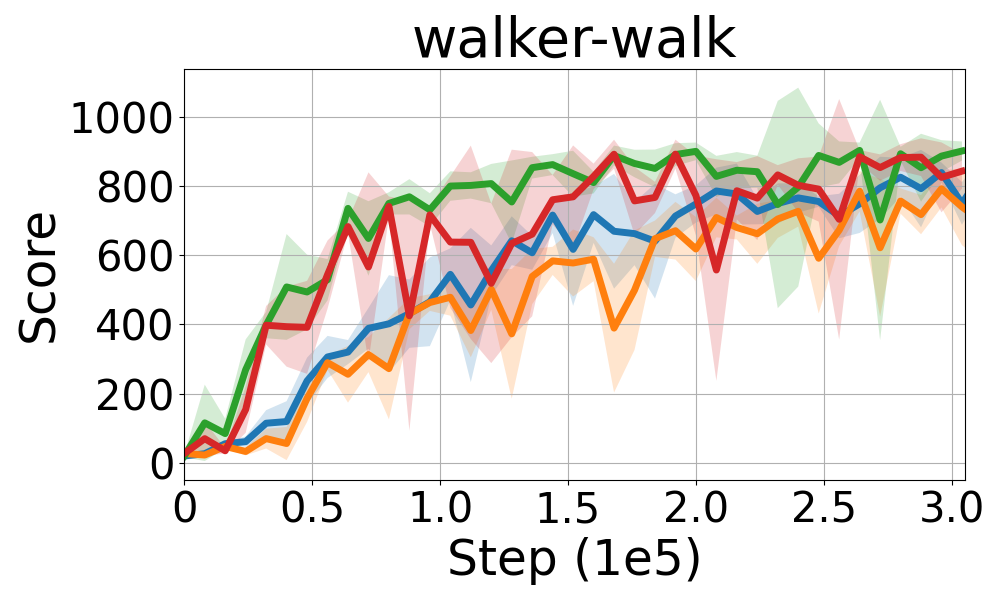}  
    \end{subfigure}
    \caption{Experimental results on various DMC tasks.}
    \vspace{-0.5em}  
    \label{fig:dmc}
\end{figure}
\noindent To evaluate the robustness and effectiveness of our method, we conducted experiments on several tasks from DMC and compared the average episode rewards of clients using our FedRAG method with \(\lambda=0.001\) and the non-federated Local method in both same and other environments, as illustrated in Figure~\ref{fig:dmc}. 
In cartpole-swing and finger-spin tasks, FedRAG significantly outperformed the Local method in other environments while maintaining near-optimal performance in the same environment. This success stems from its federated approach, which integrates global knowledge while preserving local training advantages.
In cheetah-run task, Local clients trained only on their own environments exhibited declining performance in other environments over time. In contrast, FedRAG maintained stable performance in other environments, benefiting from global knowledge. By the end of training, FedRAG outperformed the Local method in cross-environment evaluations.
In walker-walk task, FedRAG demonstrated faster convergence and higher episode rewards across all environments, benefiting from federated state projection functions that enhanced task-relevant feature extraction and generalization.
These results confirm the robustness and generalization of FedRAG across diverse tasks and environments.
\subsection{FedRAG Performance with Increasing Clients and Environmental Heterogeneity}
\label{heter}
\begin{figure}[h]
    \centering
    \vspace{-0.8em}
    \begin{subfigure}{0.48\textwidth}
    \vspace{-0.3em}
        \centering
        \includegraphics[width=\linewidth]{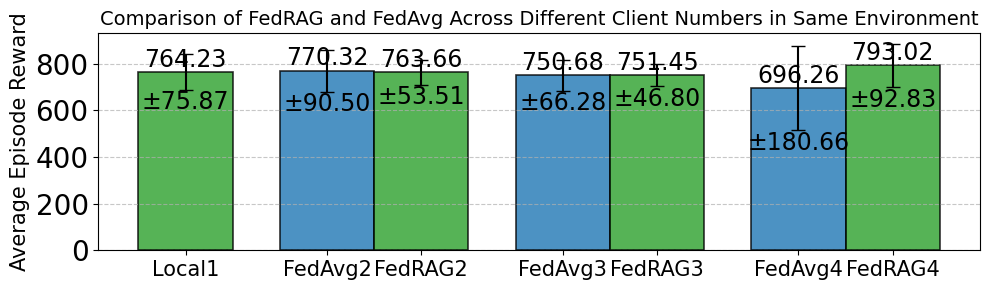}  
    \end{subfigure}
    \begin{subfigure}{0.48\textwidth}
        \centering
        \includegraphics[width=\linewidth]{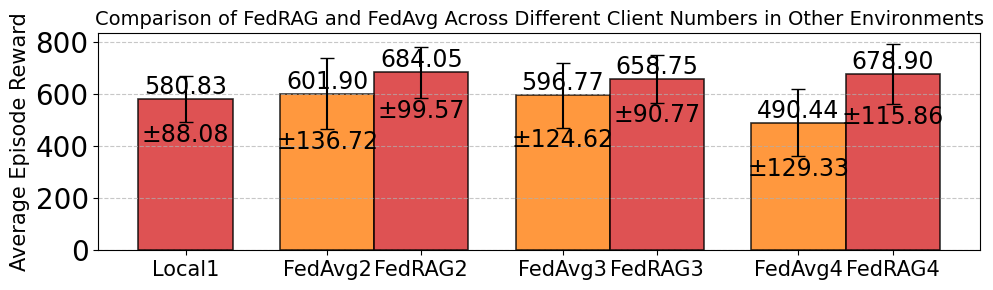}  
    \end{subfigure}
    \caption{Comparison of FedRAG and FedAvg across different client numbers in same and other environments.}
    \label{fig:fedrag_performance4}
\end{figure}
\vspace{-0.3em}
\noindent We evaluated FedRAG's performance in heterogeneous environments with increasing client numbers $N$, comparing it to FedAvg. Using $N$ pole lengths sampled uniformly from $[0.9, 1]$, we created diverse CartPole environments, with $N$ ranging from 1 to 4. 
As $N$ increased, greater environmental heterogeneity hindered policy convergence, while additional clients provided more learning information. As shown in Figure~\ref{fig:fedrag_performance4}, FedAvg's performance deteriorated significantly in other environments, while FedRAG remained stable across both same and other settings, demonstrating its robustness to heterogeneity.
This stability is attributed to FedRAG's behavior metric-based state projection, which captures task-relevant features and filters out environment-specific noise, enabling effective generalization.
The L2 regularization ensures local updates align with the global model, capturing shared dynamics across clients and effectively adapting the projection function as the client count increases.

\section{Conclusion}
Sharing the parameters of the approximated behavior metric-based state projection function enhances the performance of FRL and protects sensitive local information.
In this work, we propose FedRAG, a FRL framework that shares the parameters of the state projections among clients. 
Under the FedRAG framework, we introduce a behavioral metric-based state projection function and develop its practical approximation algorithm in Federated Learning settings.
We conduct empirical studies on several reinforcement learning tasks to verify the effectiveness of our proposed method.

\section*{Acknowledgements}
This work is supported by the Science Foundation of China University of Petroleum, Beijing (Grant No. 2462023YJRC024) and the Frontier Interdisciplinary Exploration Research Program of China University of Petroleum, Beijing (Grant No. 2462024XKQY003). Zhongqi Lu is the corresponding author.

\bibliographystyle{named}
\bibliography{ijcai25}

\begin{thebibliography}{}

\bibitem[\protect\citeauthoryear{Abadi \bgroup \em et al.\egroup }{2016}]{abadi2016deep}
Martin Abadi, Andy Chu, Ian Goodfellow, H~Brendan McMahan, Ilya Mironov, Kunal Talwar, and Li~Zhang.
\newblock Deep learning with differential privacy.
\newblock In {\em Proceedings of the 2016 ACM SIGSAC conference on computer and communications security}, pages 308--318, 2016.

\bibitem[\protect\citeauthoryear{Agarwal \bgroup \em et al.\egroup }{2021}]{agarwal2021contrastive}
Rishabh Agarwal, Marlos~C Machado, Pablo~Samuel Castro, and Marc~G Bellemare.
\newblock Contrastive behavioral similarity embeddings for generalization in reinforcement learning.
\newblock {\em arXiv preprint arXiv:2101.05265}, 2021.

\bibitem[\protect\citeauthoryear{Anwar and Raychowdhury}{2021}]{DBLP:journals/corr/abs-2103-06473}
Aqeel Anwar and Arijit Raychowdhury.
\newblock Multi-task federated reinforcement learning with adversaries.
\newblock {\em CoRR}, abs/2103.06473, 2021.

\bibitem[\protect\citeauthoryear{Castro \bgroup \em et al.\egroup }{2021}]{castro2021mico}
Pablo~Samuel Castro, Tyler Kastner, Prakash Panangaden, and Mark Rowland.
\newblock Mico: Improved representations via sampling-based state similarity for markov decision processes.
\newblock {\em Advances in Neural Information Processing Systems}, 34:30113--30126, 2021.

\bibitem[\protect\citeauthoryear{Castro}{2020}]{castro2020scalable}
Pablo~Samuel Castro.
\newblock Scalable methods for computing state similarity in deterministic markov decision processes.
\newblock In {\em Proceedings of the AAAI Conference on Artificial Intelligence}, volume~34, pages 10069--10076, 2020.

\bibitem[\protect\citeauthoryear{Chen and Pan}{2022}]{chen2022learning}
Jianda Chen and Sinno Pan.
\newblock Learning representations via a robust behavioral metric for deep reinforcement learning.
\newblock {\em Advances in Neural Information Processing Systems}, 35:36654--36666, 2022.

\bibitem[\protect\citeauthoryear{Collins \bgroup \em et al.\egroup }{2021}]{collins2021exploiting}
Liam Collins, Hamed Hassani, Aryan Mokhtari, and Sanjay Shakkottai.
\newblock Exploiting shared representations for personalized federated learning.
\newblock In {\em International conference on machine learning}, pages 2089--2099. PMLR, 2021.

\bibitem[\protect\citeauthoryear{Fallah \bgroup \em et al.\egroup }{2020}]{fallah2020personalized}
Alireza Fallah, Aryan Mokhtari, and Asuman Ozdaglar.
\newblock Personalized federated learning with theoretical guarantees: A model-agnostic meta-learning approach.
\newblock {\em Advances in neural information processing systems}, 33:3557--3568, 2020.

\bibitem[\protect\citeauthoryear{Fan \bgroup \em et al.\egroup }{2021}]{fan2021fault}
Xiaofeng Fan, Yining Ma, Zhongxiang Dai, Wei Jing, Cheston Tan, and Bryan Kian~Hsiang Low.
\newblock Fault-tolerant federated reinforcement learning with theoretical guarantee.
\newblock {\em Advances in Neural Information Processing Systems}, 34:1007--1021, 2021.

\bibitem[\protect\citeauthoryear{Fan \bgroup \em et al.\egroup }{2023}]{fan2023fedhql}
Flint~Xiaofeng Fan, Yining Ma, Zhongxiang Dai, Cheston Tan, Bryan Kian~Hsiang Low, and Roger Wattenhofer.
\newblock Fedhql: Federated heterogeneous q-learning.
\newblock {\em arXiv preprint arXiv:2301.11135}, 2023.

\bibitem[\protect\citeauthoryear{Fang and Qian}{2021}]{fang2021privacy}
Haokun Fang and Quan Qian.
\newblock Privacy preserving machine learning with homomorphic encryption and federated learning.
\newblock {\em Future Internet}, 13(4):94, 2021.

\bibitem[\protect\citeauthoryear{Ferns \bgroup \em et al.\egroup }{2011}]{ferns2011bisimulation}
Norm Ferns, Prakash Panangaden, and Doina Precup.
\newblock Bisimulation metrics for continuous markov decision processes.
\newblock {\em SIAM Journal on Computing}, 40(6):1662--1714, 2011.

\bibitem[\protect\citeauthoryear{Geyer \bgroup \em et al.\egroup }{2017}]{geyer2017differentially}
Robin~C Geyer, Tassilo Klein, and Moin Nabi.
\newblock Differentially private federated learning: A client level perspective.
\newblock {\em arXiv preprint arXiv:1712.07557}, 2017.

\bibitem[\protect\citeauthoryear{Haarnoja \bgroup \em et al.\egroup }{2018a}]{haarnoja2018soft}
Tuomas Haarnoja, Aurick Zhou, Pieter Abbeel, and Sergey Levine.
\newblock Soft actor-critic: Off-policy maximum entropy deep reinforcement learning with a stochastic actor.
\newblock In {\em International conference on machine learning}, pages 1861--1870. PMLR, 2018.

\bibitem[\protect\citeauthoryear{Haarnoja \bgroup \em et al.\egroup }{2018b}]{haarnoja2018soft2}
Tuomas Haarnoja, Aurick Zhou, Kristian Hartikainen, George Tucker, Sehoon Ha, Jie Tan, Vikash Kumar, Henry Zhu, Abhishek Gupta, Pieter Abbeel, et~al.
\newblock Soft actor-critic algorithms and applications.
\newblock {\em arXiv preprint arXiv:1812.05905}, 2018.

\bibitem[\protect\citeauthoryear{Jin \bgroup \em et al.\egroup }{2022}]{jin2022federated}
Hao Jin, Yang Peng, Wenhao Yang, Shusen Wang, and Zhihua Zhang.
\newblock Federated reinforcement learning with environment heterogeneity.
\newblock In {\em International Conference on Artificial Intelligence and Statistics}, pages 18--37. PMLR, 2022.

\bibitem[\protect\citeauthoryear{Kemertas and Aumentado-Armstrong}{2021}]{kemertas2021towards}
Mete Kemertas and Tristan Aumentado-Armstrong.
\newblock Towards robust bisimulation metric learning.
\newblock {\em Advances in Neural Information Processing Systems}, 34:4764--4777, 2021.

\bibitem[\protect\citeauthoryear{Li \bgroup \em et al.\egroup }{2020a}]{li2020federated2}
Tian Li, Anit~Kumar Sahu, Ameet Talwalkar, and Virginia Smith.
\newblock Federated learning: Challenges, methods, and future directions.
\newblock {\em IEEE signal processing magazine}, 37(3):50--60, 2020.

\bibitem[\protect\citeauthoryear{Li \bgroup \em et al.\egroup }{2020b}]{li2020federated}
Tian Li, Anit~Kumar Sahu, Manzil Zaheer, Maziar Sanjabi, Ameet Talwalkar, and Virginia Smith.
\newblock Federated optimization in heterogeneous networks.
\newblock {\em Proceedings of Machine learning and systems}, 2:429--450, 2020.

\bibitem[\protect\citeauthoryear{Li \bgroup \em et al.\egroup }{2021a}]{li2021model}
Qinbin Li, Bingsheng He, and Dawn Song.
\newblock Model-contrastive federated learning.
\newblock In {\em Proceedings of the IEEE/CVF conference on computer vision and pattern recognition}, pages 10713--10722, 2021.

\bibitem[\protect\citeauthoryear{Li \bgroup \em et al.\egroup }{2021b}]{li2021survey}
Qinbin Li, Zeyi Wen, Zhaomin Wu, Sixu Hu, Naibo Wang, Yuan Li, Xu~Liu, and Bingsheng He.
\newblock A survey on federated learning systems: Vision, hype and reality for data privacy and protection.
\newblock {\em IEEE Transactions on Knowledge and Data Engineering}, 35(4):3347--3366, 2021.

\bibitem[\protect\citeauthoryear{Liang \bgroup \em et al.\egroup }{2020}]{liang2020think}
Paul~Pu Liang, Terrance Liu, Liu Ziyin, Nicholas~B Allen, Randy~P Auerbach, David Brent, Ruslan Salakhutdinov, and Louis-Philippe Morency.
\newblock Think locally, act globally: Federated learning with local and global representations.
\newblock {\em arXiv preprint arXiv:2001.01523}, 2020.

\bibitem[\protect\citeauthoryear{Liao \bgroup \em et al.\egroup }{2023}]{article}
Weijian Liao, Zongzhang Zhang, and Yang Yu.
\newblock Policy-independent behavioral metric-based representation for deep reinforcement learning.
\newblock {\em Proceedings of the AAAI Conference on Artificial Intelligence}, 37:8746--8754, 06 2023.

\bibitem[\protect\citeauthoryear{Lyu \bgroup \em et al.\egroup }{2020}]{lyu2020threatsfederatedlearningsurvey}
Lingjuan Lyu, Han Yu, and Qiang Yang.
\newblock Threats to federated learning: A survey, 2020.

\bibitem[\protect\citeauthoryear{McMahan \bgroup \em et al.\egroup }{2017}]{mcmahan2017communication}
Brendan McMahan, Eider Moore, Daniel Ramage, Seth Hampson, and Blaise~Aguera y~Arcas.
\newblock Communication-efficient learning of deep networks from decentralized data.
\newblock In {\em Artificial intelligence and statistics}, pages 1273--1282. PMLR, 2017.

\bibitem[\protect\citeauthoryear{Mohassel and Rindal}{2018}]{mohassel2018aby3}
Payman Mohassel and Peter Rindal.
\newblock Aby3: A mixed protocol framework for machine learning.
\newblock In {\em Proceedings of the 2018 ACM SIGSAC conference on computer and communications security}, pages 35--52, 2018.

\bibitem[\protect\citeauthoryear{Mothukuri \bgroup \em et al.\egroup }{2021}]{mothukuri2021survey}
Viraaji Mothukuri, Reza~M Parizi, Seyedamin Pouriyeh, Yan Huang, Ali Dehghantanha, and Gautam Srivastava.
\newblock A survey on security and privacy of federated learning.
\newblock {\em Future Generation Computer Systems}, 115:619--640, 2021.

\bibitem[\protect\citeauthoryear{Park and Lim}{2022}]{park2022privacy}
Jaehyoung Park and Hyuk Lim.
\newblock Privacy-preserving federated learning using homomorphic encryption.
\newblock {\em Applied Sciences}, 12(2):734, 2022.

\bibitem[\protect\citeauthoryear{Qi \bgroup \em et al.\egroup }{2021}]{qi2021federated}
Jiaju Qi, Qihao Zhou, Lei Lei, and Kan Zheng.
\newblock Federated reinforcement learning: Techniques, applications, and open challenges.
\newblock {\em arXiv preprint arXiv:2108.11887}, 2021.

\bibitem[\protect\citeauthoryear{Reddi \bgroup \em et al.\egroup }{2020}]{reddi2020adaptive}
Sashank Reddi, Zachary Charles, Manzil Zaheer, Zachary Garrett, Keith Rush, Jakub Kone{\v{c}}n{\`y}, Sanjiv Kumar, and H~Brendan McMahan.
\newblock Adaptive federated optimization.
\newblock {\em arXiv preprint arXiv:2003.00295}, 2020.

\bibitem[\protect\citeauthoryear{T~Dinh \bgroup \em et al.\egroup }{2020}]{t2020personalized}
Canh T~Dinh, Nguyen Tran, and Josh Nguyen.
\newblock Personalized federated learning with moreau envelopes.
\newblock {\em Advances in neural information processing systems}, 33:21394--21405, 2020.

\bibitem[\protect\citeauthoryear{Tan \bgroup \em et al.\egroup }{2022}]{tan2022federated}
Yue Tan, Guodong Long, Jie Ma, Lu~Liu, Tianyi Zhou, and Jing Jiang.
\newblock Federated learning from pre-trained models: A contrastive learning approach.
\newblock {\em Advances in neural information processing systems}, 35:19332--19344, 2022.

\bibitem[\protect\citeauthoryear{Tang \bgroup \em et al.\egroup }{2022}]{tang2022fesac}
Fengxiao Tang, Yilin Yang, Xin Yao, Ming Zhao, and Nei Kato.
\newblock Fesac: Federated learning-based soft actor-critic traffic offloading in space-air-ground integrated network.
\newblock {\em arXiv preprint arXiv:2212.02075}, 2022.

\bibitem[\protect\citeauthoryear{Tassa \bgroup \em et al.\egroup }{2018}]{tassa2018deepmind}
Yuval Tassa, Yotam Doron, Alistair Muldal, Tom Erez, Yazhe Li, Diego de~Las Casas, David Budden, Abbas Abdolmaleki, Josh Merel, Andrew Lefrancq, et~al.
\newblock Deepmind control suite.
\newblock {\em arXiv preprint arXiv:1801.00690}, 2018.

\bibitem[\protect\citeauthoryear{Truex \bgroup \em et al.\egroup }{2019}]{truex2019hybrid}
Stacey Truex, Nathalie Baracaldo, Ali Anwar, Thomas Steinke, Heiko Ludwig, Rui Zhang, and Yi~Zhou.
\newblock A hybrid approach to privacy-preserving federated learning.
\newblock In {\em Proceedings of the 12th ACM workshop on artificial intelligence and security}, pages 1--11, 2019.

\bibitem[\protect\citeauthoryear{Wang \bgroup \em et al.\egroup }{2020}]{wang2020optimizing}
Hao Wang, Zakhary Kaplan, Di~Niu, and Baochun Li.
\newblock Optimizing federated learning on non-iid data with reinforcement learning.
\newblock In {\em IEEE INFOCOM 2020-IEEE conference on computer communications}, pages 1698--1707. IEEE, 2020.

\bibitem[\protect\citeauthoryear{Wang \bgroup \em et al.\egroup }{2024}]{wang2024turbosvm}
Mengdi Wang, Anna Bodonhelyi, Efe Bozkir, and Enkelejda Kasneci.
\newblock Turbosvm-fl: Boosting federated learning through svm aggregation for lazy clients.
\newblock In {\em Proceedings of the AAAI Conference on Artificial Intelligence}, volume~38, pages 15546--15554, 2024.

\bibitem[\protect\citeauthoryear{Wei \bgroup \em et al.\egroup }{2020}]{wei2020federated}
Kang Wei, Jun Li, Ming Ding, Chuan Ma, Howard~H Yang, Farhad Farokhi, Shi Jin, Tony~QS Quek, and H~Vincent Poor.
\newblock Federated learning with differential privacy: Algorithms and performance analysis.
\newblock {\em IEEE transactions on information forensics and security}, 15:3454--3469, 2020.

\bibitem[\protect\citeauthoryear{Yang \bgroup \em et al.\egroup }{2019a}]{10.1145/3298981}
Qiang Yang, Yang Liu, Tianjian Chen, and Yongxin Tong.
\newblock Federated machine learning: Concept and applications.
\newblock {\em ACM Trans. Intell. Syst. Technol.}, 10(2), January 2019.

\bibitem[\protect\citeauthoryear{Yang \bgroup \em et al.\egroup }{2019b}]{federatedlearning}
Qiang Yang, Yang Liu, Yong Cheng, Yan Kang, Tianjian Chen, and Han Yu.
\newblock Federated learning.
\newblock {\em Synthesis Lectures on Artificial Intelligence and Machine Learning}, 13:1--207, 12 2019.

\bibitem[\protect\citeauthoryear{Zhang \bgroup \em et al.\egroup }{2020a}]{DBLP:journals/corr/abs-2006-10742}
Amy Zhang, Rowan McAllister, Roberto Calandra, Yarin Gal, and Sergey Levine.
\newblock Learning invariant representations for reinforcement learning without reconstruction.
\newblock {\em CoRR}, abs/2006.10742, 2020.

\bibitem[\protect\citeauthoryear{Zhang \bgroup \em et al.\egroup }{2020b}]{zhang2020batchcrypt}
Chengliang Zhang, Suyi Li, Junzhe Xia, Wei Wang, Feng Yan, and Yang Liu.
\newblock $\{$BatchCrypt$\}$: Efficient homomorphic encryption for $\{$Cross-Silo$\}$ federated learning.
\newblock In {\em 2020 USENIX annual technical conference (USENIX ATC 20)}, pages 493--506, 2020.

\bibitem[\protect\citeauthoryear{Zhang \bgroup \em et al.\egroup }{2022}]{zhang2022no}
Xiaojin Zhang, Hanlin Gu, Lixin Fan, Kai Chen, and Qiang Yang.
\newblock No free lunch theorem for security and utility in federated learning.
\newblock {\em ACM Transactions on Intelligent Systems and Technology}, 14(1):1--35, 2022.

\bibitem[\protect\citeauthoryear{Zhang \bgroup \em et al.\egroup }{2023}]{zhang2023federated}
Fengda Zhang, Kun Kuang, Long Chen, Zhaoyang You, Tao Shen, Jun Xiao, Yin Zhang, Chao Wu, Fei Wu, Yueting Zhuang, et~al.
\newblock Federated unsupervised representation learning.
\newblock {\em Frontiers of Information Technology \& Electronic Engineering}, 24(8):1181--1193, 2023.

\bibitem[\protect\citeauthoryear{Zhao \bgroup \em et al.\egroup }{2018}]{zhao2018federated}
Yue Zhao, Meng Li, Liangzhen Lai, Naveen Suda, Damon Civin, and Vikas Chandra.
\newblock Federated learning with non-iid data.
\newblock {\em arXiv preprint arXiv:1806.00582}, 2018.

\bibitem[\protect\citeauthoryear{Zhu \bgroup \em et al.\egroup }{2019}]{zhu2019deepleakagegradients}
Ligeng Zhu, Zhijian Liu, and Song Han.
\newblock Deep leakage from gradients, 2019.

\bibitem[\protect\citeauthoryear{Zhuo \bgroup \em et al.\egroup }{2019}]{zhuo2019federated}
Hankz~Hankui Zhuo, Wenfeng Feng, Yufeng Lin, Qian Xu, and Qiang Yang.
\newblock Federated deep reinforcement learning.
\newblock {\em arXiv preprint arXiv:1901.08277}, 2019.

\end{thebibliography}

\appendix

\section{Experimental Details}
\label{appendix:all}
\subsection{Networks and Hyperparameters}

\begin{table}[t]
\centering
\begin{tabular}{l l}
\toprule
\textbf{Hyperparameter} & \textbf{Value} \\
\midrule
Episode length                  & 1000 \\
Training steps                  & 500,000 \\
Replay buffer capacity          & 20,000 \\
Batch size                      & 128 \\
Discount factor (\(\gamma\))    & 0.99 \\
Optimizer                       & Adam \\
Network learning rate           & \( 5 \times 10^{-4} \) \\
Log \(\alpha\) learning rate    & \( 1 \times 10^{-4} \) \\
\(\tau_\phi\)                   & 0.05 \\
\(\tau_Q\)                      & 0.01 \\
Target Q-network update frequency & 2 \\
Actor network update frequency  & 2 \\
\(\alpha_{RAG}\)                & 0.5 \\
\(\alpha_P\)                    & \( 1 \times 10^{-4} \) \\
Actor log std bound             & [-10, 2] \\
Action repeat (CartPole/Cheetah)& 8 / 4 \\
Action repeat (Finger/Walker)   & 2 \\
\bottomrule
\end{tabular}
\caption{Networks hyperparameters}
\label{sampletable}
\end{table}

Each client's Q networks include a state encoder \( \phi_\omega \), which consists of stacked convolutional layers and a fully connected layer. It processes 3 stacked frames to produce the state representation \( \phi_\omega(s) \) with input dimensions of \( 9 \times 84 \times 84 \), convolutional kernels \([3, 3, 3, 3]\), 32 channels, and strides \([2, 1, 1, 1]\), resulting in an output dimension of 100. The Q-network has three fully connected layers with 1024 hidden units, taking input from \( \phi_\omega(s) \) and action \( a \). The actor network also consists of three fully connected layers that output the policy \( \pi \). Both the dynamics model \( \hat{P}_\eta \) and the reward function \( \hat{R}_\xi \) are two-layer MLPs with 512 hidden units, using ReLU activation. This architecture efficiently generates policies and Q-values from state inputs. Other hyperparameters are listed in Table~\ref{sampletable}.

\subsection{Ablation Study on FedRAG Client Updates}
\begin{figure}[h]
\begin{center}
\includegraphics[width=\linewidth]{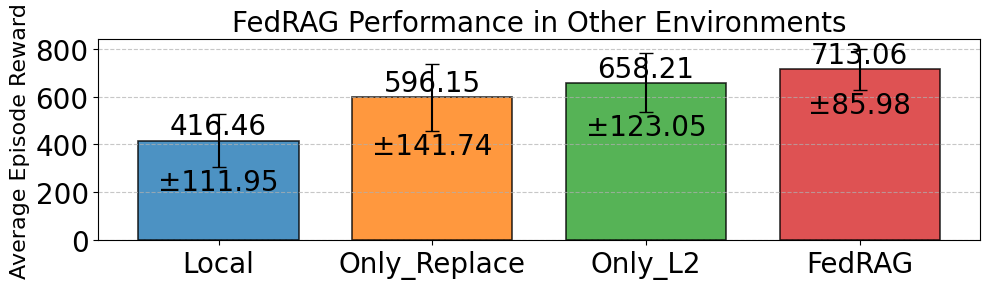}
\end{center}
\caption{Performance comparison of FedRAG and variants in other environments.}
\label{fig:fedrag_ablation}
\end{figure}
\noindent The FedRAG client update formula in Equation 15 has two key components for data sharing: replacing local parameters with global ones during distribution and applying L2 regularization to align local updates with global parameters.

To evaluate the impact of these components, we conducted ablation experiments, as shown in Figure~\ref{fig:fedrag_ablation}. We compared four approaches: Local (no federated learning), Only\_Replace (global parameters replace local ones without L2 regularization), Only\_L2 (L2 regularization without replacing local parameters), and FedRAG (both global replacement and L2 regularization). The metrics measured were the average episode reward and standard deviation in different environments.

The results show that replacing local parameters with global ones improves generalization by leveraging shared knowledge, while L2 regularization enhances robustness by preventing overfitting. Omitting either component resulted in significant performance declines, confirming their essential role in our federated learning approach.

\subsection{Distracting DeepMind Control Suite}
\begin{figure}[htbp]
    \centering
    \begin{subfigure}[b]{0.235\textwidth}
        \centering
        \includegraphics[width=0.48\textwidth]{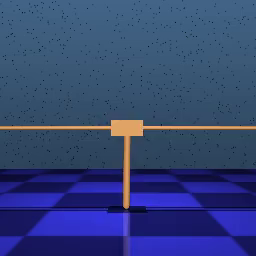}
        \includegraphics[width=0.48\textwidth]{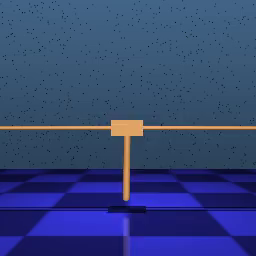}
        \caption*{(a) original background}
    \end{subfigure}
    \begin{subfigure}[b]{0.235\textwidth}
        \centering
        \includegraphics[width=0.48\textwidth]{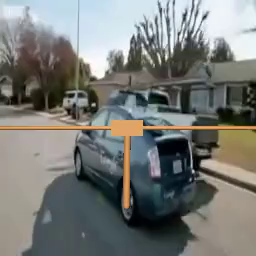}
        \includegraphics[width=0.48\textwidth]{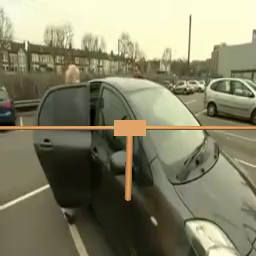}
        \caption*{(b) natural video background}
    \end{subfigure}
    \caption{Illustrations of observations in DMC cartpole-swingup task for pole lengths 1.0 and 0.9.}
    \label{fig:Observation1}
\end{figure}

\begin{figure}[h]
\begin{center}
\includegraphics[width=\linewidth]{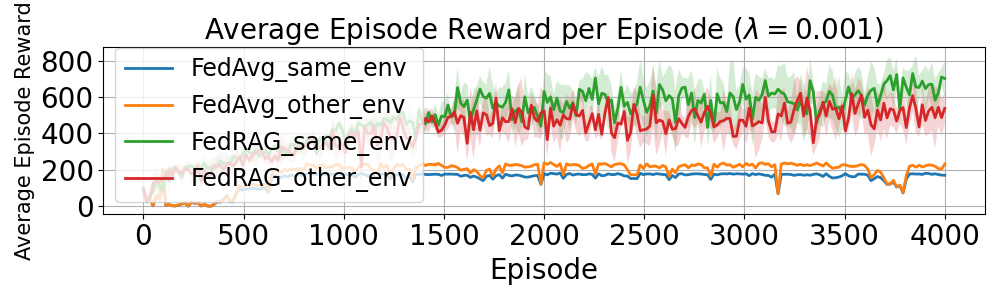}
\end{center}
\caption{Performance comparison of FedRAG and FedAvg with background distraction.}
\label{fig:fedrag_distraction}
\end{figure}

\noindent To evaluate the generalization and robustness of our method, we conducted experiments using the CartPole task in the DeepMind Control Suite, with background distractions and varying pole lengths to simulate different environments, as shown in Figure~\ref{fig:Observation1}. We replaced the background with clips from the Kinetics dataset, which serves as a distraction for the RL algorithm. We selected 1,000 continuous frames from the video dataset for training the reinforcement learning clients and evaluated them using another 1,000 frames.

The results presented in Figure~\ref{fig:fedrag_distraction} demonstrate that FedRAG outperforms FedAvg in both the same and other environments, confirming that our method is more effective at learning generalizable state representations and better at capturing task-relevant information, even in complex settings.

\subsection{Generalization Evaluation in Unseen Environments}
\begin{figure}[h]
\begin{center}
\includegraphics[width=\linewidth]{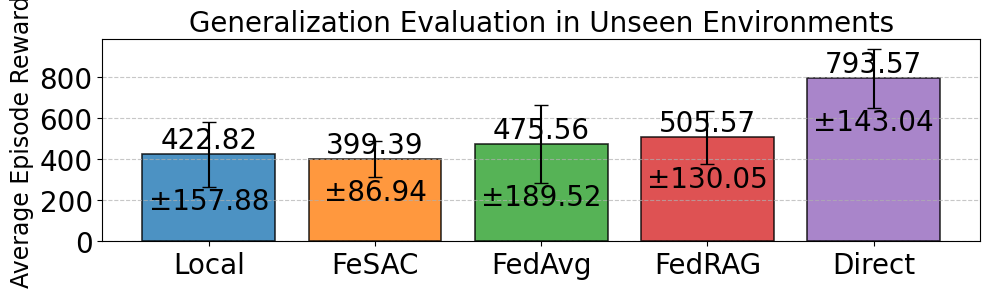}
\end{center}
\caption{Comparison of FedRAG and Baseline in unseen environment.}
\label{fig:fedrag_unseen}
\end{figure}

To further evaluate the generalization ability of our proposed FedRAG method, we took the clients trained in Section 5.2 and tested them in a completely unseen environment, where none of the clients had prior exposure. We assessed their average episode reward, and the results are shown in Figure~\ref{fig:fedrag_unseen}. Our method outperformed FedAvg and Local, achieving performance close to that of the client trained directly in the unseen environment. In contrast, FeSAC showed poor performance. These results demonstrate that our approach enables clients to generalize more effectively to new, previously unseen environments.

\subsection{Illustrations of Observations in Various DMC Tasks}
\label{app:Illu}
\begin{figure}[htbp]
    \centering
    \begin{subfigure}[b]{0.235\textwidth}
        \centering
        \includegraphics[width=0.45\textwidth]{Images/cartpole1_1.png}
        \includegraphics[width=0.45\textwidth]{Images/cartpole2_1.png}
        \caption*{(a) cartpole-swing}
    \end{subfigure}
    \hfill
    \begin{subfigure}[b]{0.235\textwidth}
        \centering
        \includegraphics[width=0.45\textwidth]{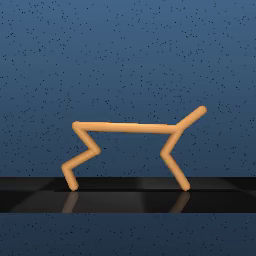}
        \includegraphics[width=0.45\textwidth]{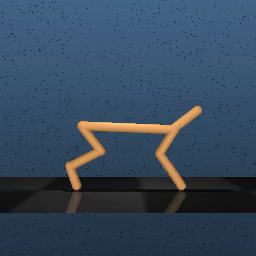}
        \caption*{(b) cheetah-run}
    \end{subfigure}
    \hfill
    \begin{subfigure}[b]{0.235\textwidth}
        \centering
        \includegraphics[width=0.45\textwidth]{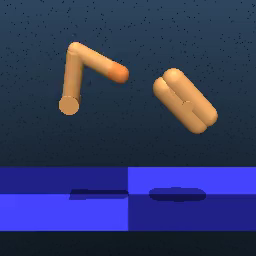}
        \includegraphics[width=0.45\textwidth]{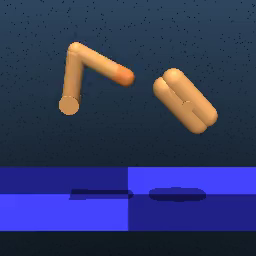}
        \caption*{(c) finger-spin}
    \end{subfigure}
    \hfill
    \begin{subfigure}[b]{0.235\textwidth}
        \centering
        \includegraphics[width=0.45\textwidth]{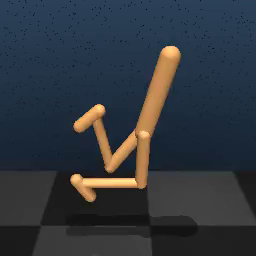}
        \includegraphics[width=0.45\textwidth]{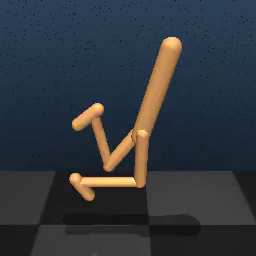}
        \caption*{(d) walker-walk}
    \end{subfigure}
    \caption{Illustrations of observations in DMC tasks for pole lengths (1.0 and 0.9), cheetah torso lengths (1.0 and 0.9), finger distal lengths (0.16 and 0.18), and walker torso lengths (0.3 and 0.35)}
    \label{fig:Observation2}
\end{figure}

\noindent As shown in Figure~\ref{fig:Observation2}, we simulated different environments by modifying key physical parameters for several tasks from the DeepMind Control Suite, including cartpole-swing, cheetah-run, finger-spin, and walker-walk. Each task has a unique goal: balancing a swinging pole in cartpole-swing, maximizing speed in cheetah-run, rotating a finger in finger-spin, and simulating bipedal locomotion in walker-walk. 

\section{Properties and Proofs of the RAG Distance}
\label{prove2}
\begin{theorem}
\label{th1}
The function $d^\pi$ is a contraction mapping with respect to the $L_\infty$ norm and has a unique fixed-point $D^\pi$.
\end{theorem}
\begin{proof}
Let $D,D' \in \mathbb{M}$. We have
\begin{align*}
& \left|d^\pi(D)(s_i, s_j) - d^{\pi}(D')(s_i, s_j)\right| \\ = & \left|\gamma\sum_{a_i,a_j}\pi(a_i|s_i)\pi(a_j|s_j)(D - D')(\mathbb{E}[s'_i],\mathbb{E}[s'_j])\right|\
\\ \leq & \gamma || D - D'||_ {\infty}.
\end{align*}

Therefore, $d^\pi$ is a contraction mapping with respect to the $L_\infty$ norm.
By Banach's fixed-point theorem, it follows that $d^\pi$ has a unique fixed point, denoted by $D^\pi$. 
\end{proof}

Theorem~\ref{th1} provides a convergence guarantee for the RAG distance that by iterating $d^\pi$, the distance $D$ will converge to the fixed-point $D^\pi$.

\begin{theorem}[Value function difference bound]
\label{th2}
Given states $s_i$ and state $s_j$, and a policy $\pi$, we have
\begin{align*}
|V^\pi(s_i)-V^\pi(s_j)| \leq D^\pi(s_i,s_j).
\end{align*}
\end{theorem}

Theorem~\ref{th2} shows that the RAG distance between two states serves as an upper bound on the difference in their value functions, ensuring that essential behavioral information is preserved.

\end{document}